\documentclass{bmcart}

\usepackage{amsmath,amsthm,amssymb,amsbsy}
\usepackage{graphicx}
\usepackage{mathtools}
\usepackage{authblk}
\usepackage{xcolor}
\usepackage{tikz} 

\newcommand{\sO}{\mathcal{O}}

\newtheorem{defi}{Definition}[section]
\newtheorem{thm}{Theorem}[section]

\newtheorem{lemma}{Lemma}[section]
\newtheorem{eg}{Example}[section]

\usepackage{algorithm}
\usepackage{algorithmic}

\DeclareMathOperator*{\argmax}{argmax} 

\usepackage[utf8]{inputenc} 

\startlocaldefs
\endlocaldefs

\begin{document}

\begin{frontmatter}

\begin{fmbox}
\dochead{Research}


\title{CausNet : Generational orderings based search for optimal Bayesian networks via dynamic programming with parent set constraints}


\author[
   addressref={aff1},                   
   corref={aff1},                       
   email={nandsh11@gmail.com}   
]{\fnm{Nand} \snm{Sharma}}
\author[
   addressref={aff1}
]{\fnm{Joshua} \snm{Millstein}}

\address[id=aff1]{
  \orgname{Division of Biostatistics, Department of Preventive Medicine, University of Southern California}, 
  \city{Los Angeles},                              
  \cny{USA}.                                   
}


\end{fmbox}


\begin{abstractbox}

\begin{abstract} 
\parttitle{Background} 
Finding a globally optimal Bayesian Network using exhaustive search is a problem with super-exponential complexity, which severely restricts the number of variables that it can work for. 
We implement a dynamic programming based algorithm with built-in dimensionality reduction and parent set identification. This reduces the search space drastically and can be applied to large-dimensional data. We use what we call ‘generational orderings’ based search for optimal networks, which is a novel way to efficiently search the space of possible networks given the possible parent sets. The algorithm supports both continuous and categorical data, and categorical as well as survival outcomes.

\parttitle{Results} 
We demonstrate the efficacy of our algorithm on both synthetic and real data. In simulations, our algorithm performs better than three state-of-art algorithms that are currently used extensively. We then apply it to an Ovarian Cancer gene expression dataset with 513 genes and a survival outcome. Our algorithm is able to find an optimal network describing the disease pathway consisting of 6 genes leading to the outcome node in a few minutes on a basic computer. 

\parttitle{Conclusions} 
Our generational orderings based search for optimal networks, is both efficient and highly scalable approach to finding optimal Bayesian Networks, that can be applied to 1000s of variables. Using specifiable parameters - correlation, FDR cutoffs, and in-degree - one can increase or decrease the number of nodes and density of the networks. Availability of two scoring option - BIC and Bge - and implementation for survival outcomes and mixed data types makes our algorithm very suitable for many types of high dimensional data in a variety of fields.
\end{abstract}




\end{abstractbox}
%

\end{frontmatter}



\section{Introduction}
Optimal Bayesian Network (BN) Structure Discovery is a method of learning Bayesian networks from data that has applications in wide variety of areas including epidemiology (see e.g. \cite{ExpressionApp}, \cite{appNeuro}, \cite{appDis1}, \cite{AppGenes}). Disease pathways found using BN leading to a phenotype outcome can lead to better understanding, diagnosis and treatment of a disease. The challenge in finding an optimal BN lies in the super-exponential complexity of the search. Dynamic programming can reduce the complexity to exponential, but still the number of features/nodes being explored remains very small -  $10$ to $20$ nodes - and only by restricting the maximum number of parents for each node. To alleviate the challenge of high dimensionality, we implement a dynamic programming algorithm with parent set constraints. We use what we call ‘generational orderings’ based search for optimal networks, which is a novel way to efficiently search the space of possible networks given the possible parent sets.

Most current algorithms do not work for both continuous and categorical nodes, and none works for survival data as the outcome node. We implement support for both continuous and categorical data, and categorical as well as survival outcomes. This is especially useful for disease modeling where mixed data  and survival outcomes are common. We also provide options of two common scoring functions and also output multiple best networks when available.

Our main novel contribution aside from providing software is the revision of the Silander algorithm 3 \cite{Silander} to incorporate possible parent sets for a much more efficient way to explore the search space as compared to the original approach based on lexicographical ordering. In doing so, we  cover the entire constrained search space without searching through networks that don't conform to the parent set constraints.

\section{Background}

In this section, we  review the necessary background in Bayesian networks and the Bayesian network structure discovery problem (for more background on these topics see, for example, \cite{darwiche_2009}, \cite{bnBookKorb}).
 
A Bayesian network (BN) is a probabilistic graphical model that consists of a labeled directed acyclic graph (DAG) in which the vertices $V = \{v_1, . . . , v_p\}$ correspond to random variables and the edges represent conditional dependence of one random variable on another. Each vertex $v_i$ is labeled with a conditional probability distribution $P(v_i | parents(v_i))$ that specifies the dependence of the variable $v_i$ on its set of parents $parents(v_i)$ in the DAG. A BN can also be viewed as a factorized representation of the joint probability distribution over the random variables and as an encoding of conditional dependence and independence assumptions.

A Bayesian network $G$ can be described as a
vector $G = (G_1,...,G_p)$ of parent sets: $G_i$ is the subset of $V$ from which there are directed edges to $v_i$. Any $G$ that is a DAG corresponds to an ordering of the nodes, given by the ordered set $\{ v_{\sigma_i} \}, i \in \{ 1,2,.. ,p\}$, where $\sigma$ is a permutation of $[p]$ - the ordered set of first $p$ natural numbers, with $\sigma(i) = \sigma_i$. A BN is said to consistent with an ordering $\{ v_{\sigma_i} \}$ if parents of $v_{\sigma_i}$ are a subset of $\{ v_{\sigma_j} \}$, if $ j < i$, i.e.  
$G_{\sigma_j} \subseteq \{ v_{\sigma_j} \}, j < i$. 

One of the main methods for BN structure learning from data uses a scoring function that assigns a real value to the quality of $G$ given the data. For finding a best network structure, we maximize this score over the space of possible networks. Note that we can have multiple best networks with the same score. Scoring functions balance goodness of fit to the data with a penalty term for model complexity. Some commonly used scoring functions are BIC/MDL \cite{bicBase}, BDeu \cite{Carvalho09scoringfunctions}, and BGe \cite{bgeBase} \cite{bgeAdd}.  
We use BIC (Bayesian information criterion) and BGe (Bayesian Gaussian  equivalent) scoring functions as two options for using Causnet. BIC is a log-likelihood (LL) score where the overfitting is avoided by using a penalty term for the number of parameters in the model, specifically $p\ln(n)$, where $n$ is the sample size. The BGe score is the posterior probability of the model hypothesis that the true distribution of the set of variables is faithful to the DAG model, meaning that it satisfies  all the conditional independencies encoded by the DAG, and is proportional to the marginal likelihood and the graphical prior\cite{bgeBase} \cite{bgeAdd}. 

\section{CausNet}
CausNet uses the dynamic programming approach to finding a best Bayesian network structure for a given dataset. The idea of using dynamic programming for exact Bayesian network structure  discovery was first proposed by Koivisto \& Sood  \cite{Koivisto2004ExactBS}, \cite{Koivisto2}. Recent work using dynamic programming, includes  that by Silander and Myllym\"{a}ki in  \cite{Silander} and by Singh and Andrew in \cite{Singh}. 

We closely follow the algorithm proposed by Silander and Myllym\"{a}ki (SM algorithm henceforth) and make heuristic modifications focusing on finding sparse networks and disease pathways. This is achieved by dimensionality reduction with parent set identification, and by restricting the search to the space of `generational' orderings rather than lexicographical ordering as in the original SM algorithm.

Finding a best Bayesian network structure is NP-hard \cite{Chickering2004}. The number of possible structures for $n$ variables is $\sO\left(  n! 2^{{n \choose 2}} \right) $, making exhaustive search impractical. So the dynamic programming algorithms are feasible only for small data sets; usually less than $20$ variables; the number of variables can be increased somewhat by using small in-degree (maximum number of parents for any node). 
But bounding the in-degree by a constant $k$ does not help much, a lower bound for number of possible graphs still being at least $n! 2^{kn \log n}$ (for large enough $n$) \cite{Koivisto2004ExactBS}. 
We introduce parent set identification and `generational' orderings based search to reduce the search space and thus scale up the dynamic programming algorithm for higher number of variables.  

The dynamical programming approach uses the key fact about DAGs that every DAG has at least one sink. The problem of finding a best Bayesian network given the data $\cal{D}$ starts with finding a best sink for the whole set of nodes. 
Then that node is removed and the process is repeated recursively for the remaining set of nodes. The result is an ordering of the nodes $\{ v_{\sigma_i} \}, i \in \{ 1,2,.. ,p\}$, from which the DAG can be recovered. 
Denoting the best sink by $s$, and the score of a best network with nodes $V$ by $ bestscore(V)$, and the best score of $s$ with parents in $U$ by $bestScore(s,U)$, the recursion is given by the following relation :
\begin{equation}
bestScore(V) = bestscore(V \setminus \left\lbrace s \right\rbrace ) +bestScore(s, V \setminus \left\lbrace s \right\rbrace ).
\end{equation}
To implement the above recursion, the idea of a local score for a node $ v_i $ with parents $parents(v_i)$ is used, which we get using a scoring function. The requirement for a score function $score(G)$ for a network $G$ is that it should be decomposable, meaning that the total score of the network is the sum of scores for each node in the network, and the score of a node depends only on the node and its parents. Formally,
\begin{equation}
score(G) = \sum_{i=1}^{p} localscore(v_i,G_i).
\end{equation}
where the local scoring function $localscore(x,y)$ gives the score of $x$ with parents $y$ in the network $G$. In a given set of possible parents $pp_i$ for node $v_i$, we find the best set of parents $bps_i$ which give the best local score for $v_i$, so that
\begin{equation}
bestScore\left( v_i, pp_i )\right) = \max_{g \subseteq pp_i} localscore(v_i, g),
\end{equation}
\begin{equation}
bps_i (pp_i ) = \argmax_{ g \subseteq pp_i  } localscore(v_i, g).
\end{equation}
Now the best sink $s$ can be found by Eq. \ref{bestSinkV}, and the best score for a best network in $V$ can be found by Eq. \ref{bestScoreV}.
\begin{equation}\label{bestSinkV}
bestSink \left( V\right) = \argmax_{s \in V} bestscore(V \setminus \left\lbrace s \right \rbrace ) + bestScore(s, V \setminus \left\lbrace s \right\rbrace ).
\end{equation}
\begin{equation}\label{bestScoreV}
bestscore(V) = \max_{s \in V}   bestscore(V \setminus \left\lbrace s \right\rbrace ) + bestScore(s, V \setminus \left\lbrace s \right \rbrace )  .
\end{equation}

\begin{figure}[ht]
    
    \begin{center}
    \begin{tikzpicture}[scale=0.7]
        \tikzstyle{every node} = [rectangle]
        \node (s) at (0,0) {$\{1,2,3,4\}$};
        \node (s123) at (-3,2) {$\{1,2,3\}$};
        \node (s124) at (-1,2) {$\{1,2,4\}$};
        \node (s134) at (1,2) {$\{1,3,4\}$};
        \node (s234) at (3,2) {$\{2,3,4\}$};
        \node (s12) at (-5,4) {$\{1,2\}$};
        \node (s13) at (-3,4) {$\{1,3\}$};
        \node (s23) at (-1,4) {$\{2,3\}$};
        \node (s14) at (1,4) {$\{1,4\}$};
        \node (s24) at (3,4) {$\{2,4\}$};
        \node (s34) at (5,4) {$\{3,4\}$};
        
        \node (s1) at (-3,6) {$\{1\}$};
        \node (s2) at (-1,6) {$\{2\}$};
        \node (s3) at (1,6) {$\{3\}$};
        \node (s4) at (3,6) {$\{4\}$};
        
        \node (t) at (0,8) {$\{ \}$};
        
        \foreach \from/\to in {s123/s, s124/s, s134/s, s234/s,
        s12/s123, s13/s123, s23/s123,
        s12/s124, s14/s124, s24/s124,
        s14/s134, s13/s134, s34/s134,
        s23/s234, s24/s234, s34/s234,
	    s1/s12,s1/s13,s1/s14,        
        s2/s12,s2/s23,s2/s24, 
        s3/s13,s3/s23,s3/s34, 
        s4/s14,s4/s24,s4/s34,
        t/s1,t/s2,t/s3,t/s4}
            \draw[->] (\from) -- (\to);
    \end{tikzpicture}
    \end{center}
    \caption{Subset lattice on a network with four nodes $\{1,2,3,4\}$.}
    \label{lattice}
\end{figure}

In Figure \ref{lattice}, the subset lattice shows all the paths that need to be searched to find a best network. Observe that each edge in the lattice encodes a sink, so that each path also encodes an ordering on the $4$ variables, e.g. the rightmost path encodes the reverse-ordering $\{1,2,3,4\}$. There are a total of $4!$ paths/orderings to be searched. Now suppose we knew the best score corresponding to each edge in all the paths, meaning the best score for the sink corresponding to that edge with the best parents from the subset at the source of that edge. Then naive depth/width first search would compare the sum of scores along all paths to get the best network. In dynamic programming, we proceed from the top of the lattice. Ignoring the empty set, we start with finding the best sink for each singleton in the first row which trivially is the singleton itself. Next, we find the best sink for the subsets of cardinality $2$ in the second row using the edge best scores. And we continue all the way down. Suppose we get the best sinks sublattice as in Figure \ref{sublattice}, then the best network is given by the only fully connected path, and corresponds to the reverse-ordering $\{4,1,2,3\}$. 

\begin{figure}[ht]
    \begin{center}
    \begin{tikzpicture}[scale=0.7]
        \tikzstyle{every node} = [rectangle]
        \node (s) at (0,0) {$\{1,2,3,4\}$};
        \node (s123) at (-3,2) {$\{1,2,3\}$};
        \node (s124) at (-1,2) {$\{1,2,4\}$};
        \node (s134) at (1,2) {$\{1,3,4\}$};
        \node (s234) at (3,2) {$\{2,3,4\}$};
        \node (s12) at (-5,4) {$\{1,2\}$};
        \node (s13) at (-3,4) {$\{1,3\}$};
        \node (s23) at (-1,4) {$\{2,3\}$};
        \node (s14) at (1,4) {$\{1,4\}$};
        \node (s24) at (3,4) {$\{2,4\}$};
        \node (s34) at (5,4) {$\{3,4\}$};
        
        \node (s1) at (-3,6) {$\{1\}$};
        \node (s2) at (-1,6) {$\{2\}$};
        \node (s3) at (1,6) {$\{3\}$};
        \node (s4) at (3,6) {$\{4\}$};
        
        \node (t) at (0,8) {$\{ \}$};
        
        \foreach \from/\to in {s123/s,
        s23/s123,
        s12/s124, 
        s13/s134, 
        s34/s234,
	    s1/s12, s1/s13, s4/s14,     
        s2/s24, 
        s3/s23, 
        s4/s34,
        t/s1,t/s2,t/s3,t/s4}
            \draw[->] (\from) -- (\to);
            \draw [red,  thick] (t) -> (s3);
            \draw [red,  thick] (s3) -> (s23);
            \draw [red,  thick] (s23) -> (s123);
            \draw [red,  thick] (s123) -> (s);
    \end{tikzpicture}
    \end{center}
    \caption{Subset sublattice of four nodes $\{1,2,3,4\}$. The arrows indicate the best sink for each subset. The red arrows indicate the ordering of best sinks providing the best network.}
    \label{sublattice}
\end{figure}

Now, using the dynamic programming approach of SM algorithm, finding the best Bayesian network structure using basic CausNet, without restricting the search space, has the following five steps. 

\begin{enumerate}
\item Calculate the local scores for all $p2^{p-1}$ different (variable, variable set)-pairs.
\item Using the local scores, find best parents for all  $p2^{p-1}$ (variable, possible parent set)-pairs.
\item Find the best sink for all $2^p$ variable sets.
\item Using the results from Step 3, find a best ordering
of the variables.
\item Find a best network using results computed in Steps 2 and 4.
\end{enumerate}

The extensions in this base version of Causnet - possible parent sets identification, 
phenotype driven search, and search space based on `Generational orderings' reduce the search space.

\subsection{Possible parent sets identification}
The possible parent set $pp_i  $ for each node $V_i$, such that $pp_i  \subseteq U_i \subseteq V \setminus \{V_i \}$,
is determined in a preliminary step using marginal association. For this, we test for pairwise association between variables using Pearson's product moment correlation coefficient. Either a specifiable $p$-value $\alpha$ is used as FDR cut-off to identify associations between pairs or we can  use correlation value cutoffs. This gives a possible parent set $pp_i $ for each node $V_i$. If a survival outcome variable is involved, we find associations with it using Cox-regression, one variable at a time.

\subsection{Phenotype driven search}
While associations between all pairs are considered as the default, we consider only two or three levels of associations starting with the outcome variable if phenotype driven search is selected, i.e. we identify the parents, `grandparents' and `great-grandparents' of the phenotype outcome.  

Collecting all the variables that have associations up to the threshold gives us the ``feasible set'' (\textit {feasSet}) of nodes. The original data is then reduced to the \textit{feasSetData}, which includes only the  \textit{feasSet} variables. This is implemented as in algorithm \ref{al_feasSet}. After algorithm \ref{al_feasSet}, the dimension of data is reduced to $\bar{p}$, $\bar{p}<p$, where 
$\bar{p}$ is the number of nodes in the $feasSet$. 

\begin{algorithm}[!h]
\caption{Find pp, Compute feasSet and feasSetData}\label{al_feasSet}
 \textbf{Input} : Data, $\alpha$, phenotypeBased, pp
\begin{algorithmic}[1]
\IF{pp Not NULL} 
\STATE pp = pp
\ELSE 
\IF{phenotypeBased =True} 
\STATE find pp for the phenotype output
\STATE find pp for the phenotype's pp
\ELSE 
\STATE find pp for all variables
\ENDIF
\ENDIF
\STATE find possible offsprings (po) for all variables
\STATE find the feasible set of variables $feasSet$
\STATE get the reduced dimensional data $feasSetData$
\end{algorithmic}
\textbf{Output} :  pp, po, feasSet and feasSetData
\end{algorithm}

\subsection{Dynamic programming on the space of `Generational orderings' of nodes}
After the possible parent sets are identified, the next step is to compute local scores for $feasSet$ nodes as shown in algorithm \ref{al_local}. Computing local scores has computational complexity $\sO\left(  \bar{p}2^{\bar{p}-1}\right) $  if there was no possible parent sets identified for nodes, but reduces to $\sO\left( \bar{p}r^{d}\right) $, where $r$ is the maximum cardinality of possible parent sets of all nodes and $d$ is the in-degree. Because of bounded indegree, the step at line $3$ in this algorithm is truncated at cardinality indegree.

\begin{algorithm}[!h]
\caption{Compute local scores for feasSet nodes}
\label{al_local}
\textbf{Input} : feasSetData, pp, indegree
\begin{algorithmic}[1]
\FOR{$v_i$ in $feasSet$} 
\STATE find all parent subsets $\{pps_j\}$ of possible parents set $pp_i$  of the node $v_i$
\STATE compute local score for node $v_i$ with parents $pps_j$
\ENDFOR
\end{algorithmic}
\textbf{Output} :  pps, ppss
\end{algorithm}

\begin{algorithm}[!h]
\caption{Compute best scores and best parents for feasSet nodes in all parent subsets}
\label{al_bps}
\textbf{Input} : feasSetData, pps, ppss, indegree
\begin{algorithmic}[1]
\FOR{$v_i$ in $feasSet$} 

\FOR{$p_{ij}$ in $pps_i$} 
\STATE find all subsets $\{ppsSub_{ijk}\}$ upto cardinality indegree
\STATE Compute best scores $bpss_ij$ and best parents $bps_{ij}$ of $v_i$ for parent subset $pps_{ij}$ from among $\{ppsSub_{ijk}\}$ using local scores in $ppss$
\ENDFOR

\ENDFOR

\end{algorithmic}
\textbf{Output} :  pps, ppss, bps, bpss
\end{algorithm}

The next step is to compute the best scores and best parents for feasSet nodes in all possible parent subsets as shown in algorithm \ref{al_bps}. This uses local scores already calculated in algorithm \ref{al_local}. This has  computational complexity $\sO\left(  \bar{p}r2^{r-1}\right) $, where $r$ is the maximum cardinality of possible parent sets of all nodes. 

Then we compute best sinks for all possible subsets of feasSet nodes restricted by the parent set constraints. Now compared with SM algorithm that explores all orderings, we explore only the space of what we call \textit{`generational orderings'} to find the best BNs. 

\begin{defi}{A generational ordering is an ordering such that each variable in the ordering has atleast one possible parent in the set of variables preceding it in the ordering, consistent with the possible parent sets. }
\end{defi}

Starting with $\bar{p}$ networks of single nodes, we add one offspring at a time, and iterate over all nodes in the set to find a best sink in a subset of cardinality increasing from $1$ to $\bar{p}$, as in algorithm \ref{al_bsinks}. Observe that while finding best sink for a subset of cardinalty $k$ at level $k$, we already have the best networks of cardinalty $k-1$ at level $k-1$, which is the key feature of the DP approach. Once we have these best sinks, we compute the reverse ordering (possibly multiple orderings) of $\bar{p}$ nodes and compute best network as shown in algorithm \ref{al_bnets}.

\begin{algorithm}[!h]
\caption{Compute best sinks for feasSet subsets}
\label{al_bsinks}
\textbf{Input} : feasSetData, po, pps, ppss, bps
\begin{algorithmic}[1]

\STATE Start with $\bar{p}$ networks of single nodes, where $\bar{p}$  is the number of nodes in $feasSet$, with their NULL scores

\STATE Add one offspring at a time, and iterate over all nodes in the set to find a best sink using information about best score/parent set from algorithm \ref{al_bps} ; end when number of nodes in the iteration is $\bar{p}$
\RETURN list of all possible $2^{\bar{p}}$ subsets of $feasSet$ with best sink and best network score for that subset of nodes
\end{algorithmic}
\textbf{Output} :  bsinks

\end{algorithm}

\begin{algorithm}[!htb]
\caption{Find the best networks}
\label{al_bnets}
\textbf{Input} : bsinks

\begin{algorithmic}[1]
\STATE Find the reverse ordering (possibly multiple orderings) of $\bar{p}$ nodes using the list from algorithm \ref{al_bsinks} 

\STATE Compute best network(s) using the reverse ordering(s) of $\bar{p}$ nodes and using the best parent set for each node in a possible parent set found in algorithm \ref{al_bps} 

\end{algorithmic}
\textbf{Output} :  bestNetwork(s)
\end{algorithm}

\begin{thm} \label{thmGeneralAll}
Without parent set and in-degree restrictions, the Causnet -the generational ordering based DP algorithm - explores all the $\sO\left(  p! 2^{{p \choose 2}} \right) $ network structures for $p$ nodes.
\end{thm}

\begin{proof}
Without parent set restrictions, every node is s possible parent of every other node. In Figure \ref{lattice}, let $k, 0 \leq k \leq p$ be the cardinality of subsets in the subset lattice for $p$ nodes. Let each row in the subset lattice be the $k$th level in the lattice. Now adding a new element in algorithm \ref{al_bsinks} corresponds to an edge between a subset of cardinality $k$ and $k-1$, which considers the added element as a sink in the subset of cardinality $k$. Number of edges to a subset of cardinality $k$ from subsets of cardinality $k-1$ is given by${k \choose k-1}$.
The number of possible parent combinations for a sink in subset of cardinality $k$, without in-degree restrictions, is given by $2^{k-1}$.
In algorithm \ref{al_bps}, we explore all these possible parent sets to find the best parents for each sink $s$ in each subset at each level $k$. The algorithm \ref{al_bsinks} uses this information to get the best sink(possibly multiple) for each subset at level $k$. The total number of networks thus searched by the Causnet algorithm is given by -

\begin{center}
\begin{math}
\begin{aligned}
& \prod_{k=1}^{p} {k \choose k-1}  2^{k-1} \\
&= {1 \choose 0}{2 \choose 1} ... {p-1 \choose p-2}{p \choose p-1} 2^{0}2^{1}...2^{p-1} \\
&= p! 2^{{p \choose 2}} .
\end{aligned}
\end{math}
\end{center}

The number of network structures is $\sO\left(  p! 2^{{p \choose 2}} \right) $ because there are many repeated structures in this combinatorial computation; e.g. there are $p!$ structures with all $p$ nodes disconnected.
\end{proof}

Now suppose the possible parent sets for the four nodes $\{ 1,2,3,4\}$are as follows : $pp_1 = \{ 2,4\}$,
$pp_2=\{ 3,1\}$, $pp_3=\{ 2\}$,$pp_4=\{ 1\}$. Factoring in these possible parent sets , the  subset lattices corresponding to those in figures \ref{lattice} and \ref{sublattice} reduce to those in Figure \ref{latticePP}.Here, the top graph shows the subset lattice with parent set restrictions, with the blue arrows encoding the best sinks for each subset. 
The bottom lattice retains only the subsets that are in the complete path from the top of the lattice to its bottom, and discards the remaining paths. These full paths represent what we call \textit{complete generational orderings}. This is how generational ordering of Causnet ensures maximum connectivity among the reduced set of $\hat{p}$ nodes in the \textit{feasSet}. The red arrows which at base are blue as well, represent the best network. 

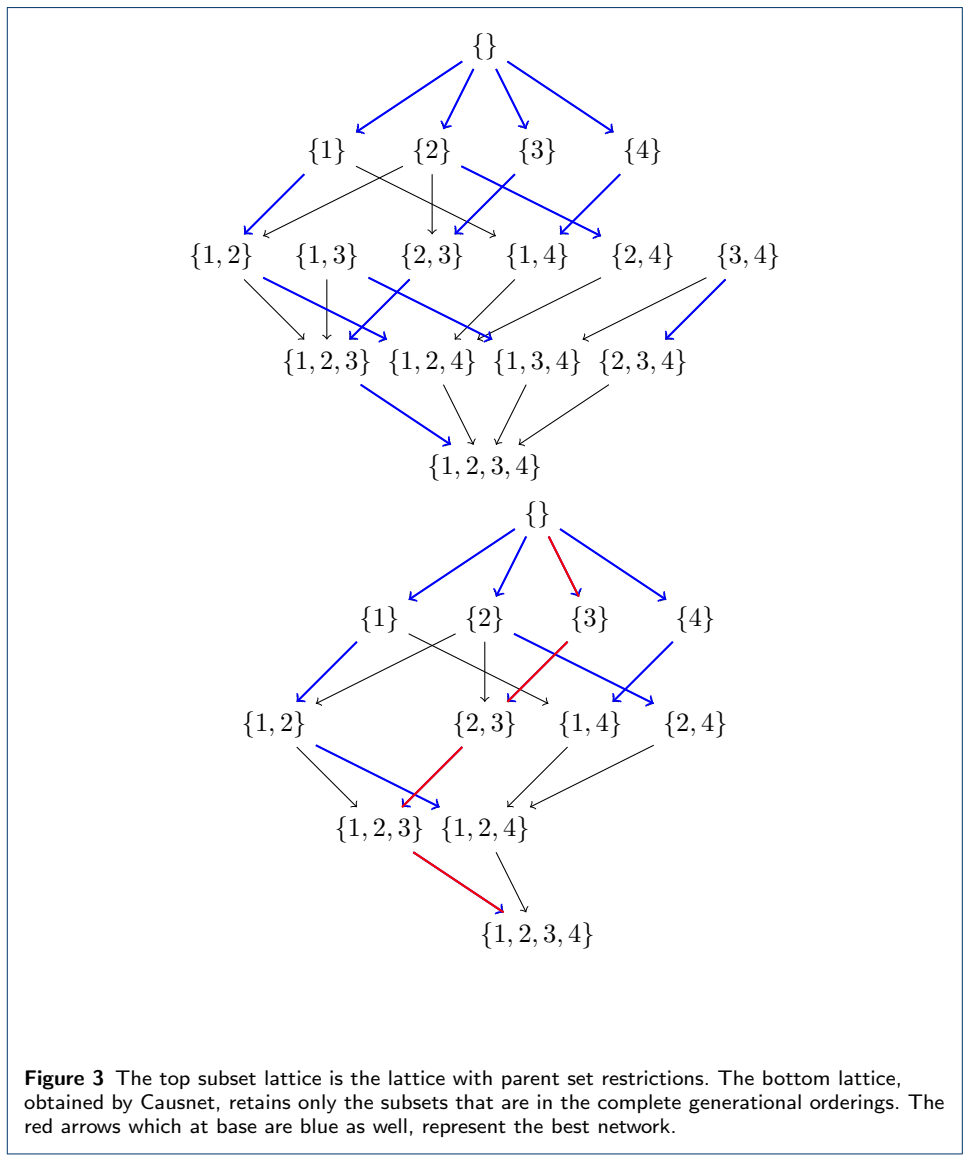
\begin{figure}[!ht]
    \begin{center}
    \begin{tikzpicture}[scale=0.7]
        \tikzstyle{every node} = [rectangle]
        \node (s) at (0,0) {$\{1,2,3,4\}$};
        \node (s123) at (-3,2) {$\{1,2,3\}$};
        \node (s124) at (-1,2) {$\{1,2,4\}$};
        \node (s134) at (1,2) {$\{1,3,4\}$};
        \node (s234) at (3,2) {$\{2,3,4\}$};
        \node (s12) at (-5,4) {$\{1,2\}$};
        \node (s13) at (-3,4) {$\{1,3\}$};
        \node (s23) at (-1,4) {$\{2,3\}$};
        \node (s14) at (1,4) {$\{1,4\}$};
        \node (s24) at (3,4) {$\{2,4\}$};
        \node (s34) at (5,4) {$\{3,4\}$};
        
        \node (s1) at (-3,6) {$\{1\}$};
        \node (s2) at (-1,6) {$\{2\}$};
        \node (s3) at (1,6) {$\{3\}$};
        \node (s4) at (3,6) {$\{4\}$};
        
        \node (t) at (0,8) {$\{ \}$};
        
        \foreach \from/\to in {s123/s, s124/s, s134/s, s234/s,
        s12/s123, s13/s123, s23/s123,
        s12/s124, s14/s124, s24/s124,
         s13/s134, s34/s134,
         s34/s234,
	    s1/s12,s1/s14,        
        s2/s12,s2/s23,s2/s24, 
        s3/s23, 
        s4/s14,
        t/s1,t/s2,t/s3,t/s4}
            \draw[->] (\from) -- (\to);
		\foreach \from/\to in         
        { s123/s,  s23/s123,
        s12/s124, 
        s13/s134, 
        s34/s234,
	    s1/s12,  s4/s14,     
        s2/s24, 
        s3/s23, 
        t/s1,t/s2,t/s3,t/s4}\draw [->] [blue,  thick] (\from) -> (\to);

    \end{tikzpicture}
    \vspace{1cm}
    \begin{tikzpicture}[scale=0.7]
        
        \tikzstyle{every node} = [rectangle]
        \node (s) at (0,0) {$\{1,2,3,4\}$};
        \node (s123) at (-3,2) {$\{1,2,3\}$};
        \node (s124) at (-1,2) {$\{1,2,4\}$};
        \node (s12) at (-5,4) {$\{1,2\}$};
        \node (s23) at (-1,4) {$\{2,3\}$};
        \node (s14) at (1,4) {$\{1,4\}$};
        \node (s24) at (3,4) {$\{2,4\}$};
        
        \node (s1) at (-3,6) {$\{1\}$};
        \node (s2) at (-1,6) {$\{2\}$};
        \node (s3) at (1,6) {$\{3\}$};
        \node (s4) at (3,6) {$\{4\}$};
        
        \node (t) at (0,8) {$\{ \}$};
        
        \foreach \from/\to in {s123/s, s124/s,
        s12/s123,  s23/s123,
        s12/s124, s14/s124, s24/s124,
        s1/s12,s1/s14,        
        s2/s12,s2/s23,s2/s24, 
        s3/s23, 
        s4/s14,
        t/s1,t/s2,t/s3,t/s4}
            \draw[->] (\from) ->(\to);
		\foreach \from/\to in         
        { s123/s, 
        s23/s123,
        s12/s124, 
        s1/s12,  s4/s14,     
        s2/s24, 
        s3/s23, 
        t/s1,t/s2,t/s3,t/s4}\draw [->][blue,  thick] (\from) -> (\to);   
		\draw [red,  thick] (t) -> (s3);
            \draw [red,  thick] (s3) -> (s23);
            \draw [red,  thick] (s23) -> (s123);
            \draw [red,  thick] (s123) -> (s);        
         \end{tikzpicture}
    \end{center}
    \caption{The top subset lattice is the lattice with parent set restrictions. The bottom lattice, obtained by Causnet, retains only the subsets that are in the complete generational orderings. The red arrows which at base are blue as well, represent the best network.}
    \label{latticePP}
\end{figure}

\begin{defi}{A generational ordering is a complete generational ordering if it has all the variables in the $feasSet$ in the ordering.}
\end{defi}

\begin{eg}
In the top subset lattice in Figure \ref{latticePP}, the two paths from subset $\{3,4\}$ downwards can be seen as two generational orderings missing a generational order relation between nodes $3$ and $4$. Let's denote this  ordering as $\{2,1\} $, which is not a complete generational ordering.
\end{eg}

\begin{lemma} \label{lemmaGenComplete}
  With parent set restrictions, the Causnet -the generational ordering based DP algorithm - searches the whole space of complete generational orderings.
\end{lemma}

\begin{proof}
To see this, suppose not. Then there is a complete generational ordering that is not searched. But the algorithm \ref{al_bsinks} adds a variable at level $k$ that is a possible offspring of the preceding subset at level $k-1$ at each step starting with the empty set. So, this missed ordering must have a variable at level $k$ in the ordering that has no possible parent in the set of variables before it at level $k-1$. That makes it an incomplete generational ordering, which is a contradiction.
\end{proof} 

\begin{thm}
With parent set restrictions, the Causnet algorithm explores all the network structures consistent with possible parent sets for $p$ nodes.
\end{thm}

\begin{proof}
Combining lemma \ref{lemmaGenComplete}, and theorem \ref{thmGeneralAll}, it's straightforward to see that Causnet discards only the networks that are either inconsistent with parent set restrictions or in an incomplete generational ordering. In each complete generational ordering, it searches all $2^{\hat{k}}$ possible parent combinations for each sink at level $k$, where $ \hat{k} \leq k-1$ because of parent set restrictions.
\end{proof}

\section{Simulations}
We compare CausNet with three other methods that have been widely used for optimal Bayesian network identification to infer disease pathways from multiscale genomics data.
The first method is Bartlett and Cussens' GOBNILP \cite{gob}, an integer learning based method that's considered state-of-art exact method for finding optimal Bayesian network. The other two methods are BNlearn's Hill Climbing (HC) and Max-min Hill Climbing (MMHC) (\cite{mmhc}, \cite{jstatsoft09}), which are both widely used approximate methods, see e.g. \cite{BN_Ainsworth}, \cite{bnlearn_Scutari}. Hill-Climbing (HC) is a score-based algorithm that uses greedy search on the space of the directed graphs  \cite{HCref}. 
Max-Min Hill-Climbing (MMHC) is a hybrid algorithm \cite{MMHCCref} that first learns the
undirected skeleton of a graph using a constraint-based algorithm called Max-Min
Parents and Children (MMPC); this is followed by the application of a score-based
search to orient the edges. 

We simulated Bayesian networks by generating an $N$ x $p$ data matrix of continuous Gaussian data. The dependencies are simulated using linear regression with the option to control effect sizes. Some number of the $p$ nodes were designated as sources ($p_1$), some intermediate ($p_2$), and some sinks ($p_3$), the remainder ($p_0$) being completely independent.  The actual DAGs of the $p_1+p_2+p_3$ nodes vary across replicates.  The requirement for being a DAG is implemented using the idea of ordering of vertices. We pick a random ordering of a randomly chosen subset of $p$ vertices. Then enforcing each vertex to have parents only from the set of vertices above itself in the ordering guarantees a DAG. 

The False Discovery Rate (FDR) and Hamming Distance are used as the metrics to compare the methods. With $FP$ defined as the number of false positives and  $TP$ defined as the number of true positives, FDR is defined as :
$$
FDR = \frac{FP}{FP+TP}.
$$
Controlling for the false discovery rate (FDR) is a way to identify as many significant features (edges in case of BNs) as possible while incurring a relatively low proportion of false positives. This is especially useful metric for high dimensional data and for network analysis (\cite{fdr}). 

The Hamming distance between two labeled graphs $G_1$ and $G_2$ is given by 
$|\lbrace \left( \left(  e \in E(G_1)  
\&  \left(e  \not\in E(G_2)\right) \right) or \left( (e \not\in E(G_1) \& e \in E(G_2)\right)\right) \rbrace |$, where $E(G_i)$ is the edge set of graph $G_i$.
Simply put, this is the number of addition/deletion operations required to turn the edge set of $G_1$ into that of $G_2$. The Hamming distance is a measure of structural similarity, and forms a metric on the space of graphs (simple or directed), and gives a good measure of goodness of a predicted graph (\cite{hamm1}, \cite{hamm2}). In the context of predicted and the truth graph, with $FP$ defined as the number of false positives and  $FN$  as the number of false negatives, Hamming Distance is defined as :
$$
Hamming Distance = FP+FN.
$$

As the first set of simulations using parent set identification using correlation tests, we ran simulations using multiple replicates of networks, the first with $p = 10,20,40,50,60,100$, and $N = 500,1000,2000$. Figure \ref{FDRPic} shows the plot of average FDR for different values of $p$ and $N$, and their linear trend with BIC scoring for CausNet. 

\begin{figure}[!htb]
\begin{center}
\includegraphics[width=4in]{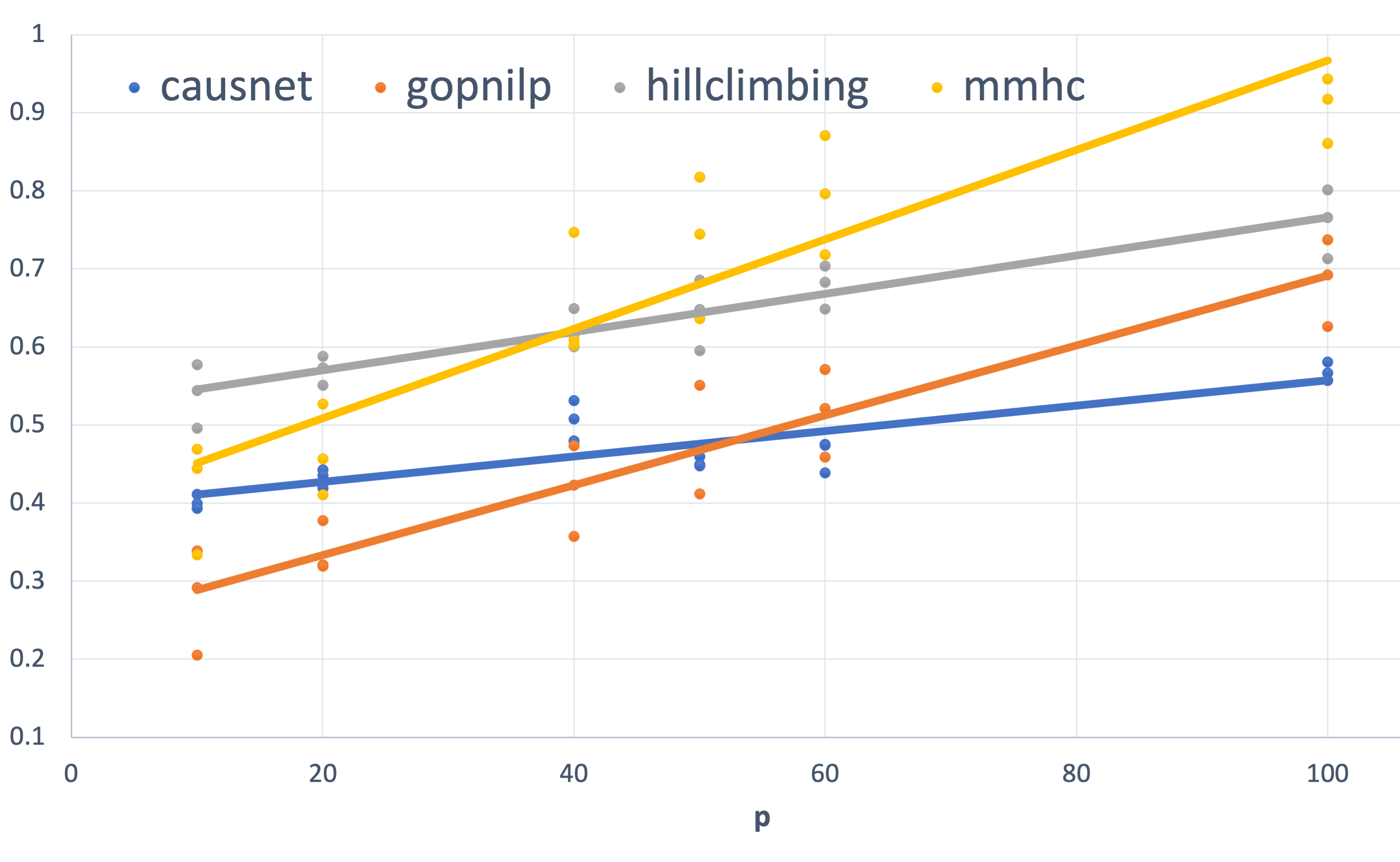}
\end{center}
\caption{\csentence{Average FDR - upto $100$ variables.}
$p = 10,20,40, 50, 60, 80, 100$ and $N = 500,1000,2000$.}

\label{FDRPic}
\end{figure} 

We can see that for lower values of $p$, Gobnilp has the lowest values of FDR with CausNet the second best, but CausNet performs the best for higher values of $p$. The results for the BGe scoring for CausNet are similar qualitatively. In the tables \ref{table:fdr1} and \ref{table:fdr2}, we show the average $FDR$ across the 9 combinations of $N$ and $p$, both with and without taking directionality into account. The results for CausNet with the choice of scoring function---either BGE or BIC---are given. The results show that our method performs very well compared with the methods considered.

\begin{table}[!ht]
\caption{FDR ($p=10,20,40, N= 500,1000,2000$)} 
\centering 
\begin{tabular}{c c c c} 
\hline\hline 
Method & FDR(undirected) & FDR(directed) \\ [0.5ex] 
\hline 
CausNet (BIC)& 0.229 & 0.412  \\
CausNet (BGE)& \bf 0.223 & 0.404  \\
Gobnilp & 0.312 & \bf 0.345  \\
BN-HC & 0.466 & 0.577  \\
BN-MMHC & 0.368 & 0.511 \\ [1ex] 
\hline 
\end{tabular}
\label{table:fdr1} 
\end{table}

\begin{table}[!ht]
\caption{FDR ($p=50,60,100, N= 500,1000,2000$)} 
\centering 
\begin{tabular}{c c c c} 
\hline\hline 
Method & FDR(undirected) & FDR(directed) \\ [0.5ex] 
\hline 
CausNet (BIC)& 0.359 & 0.494 \\
CausNet (BGE)& \bf 0.328 & \bf 0.466   \\
Gobnilp & 0.540 &  0.560  \\
BN-HC & 0.635 & 0.694 \\
BN-MMHC & 0.787 & 0.812 \\ [1ex] 
\hline 
\end{tabular}
\label{table:fdr2} 
\end{table}

\begin{figure}[!h]
\begin{center}
\includegraphics[width=4in]{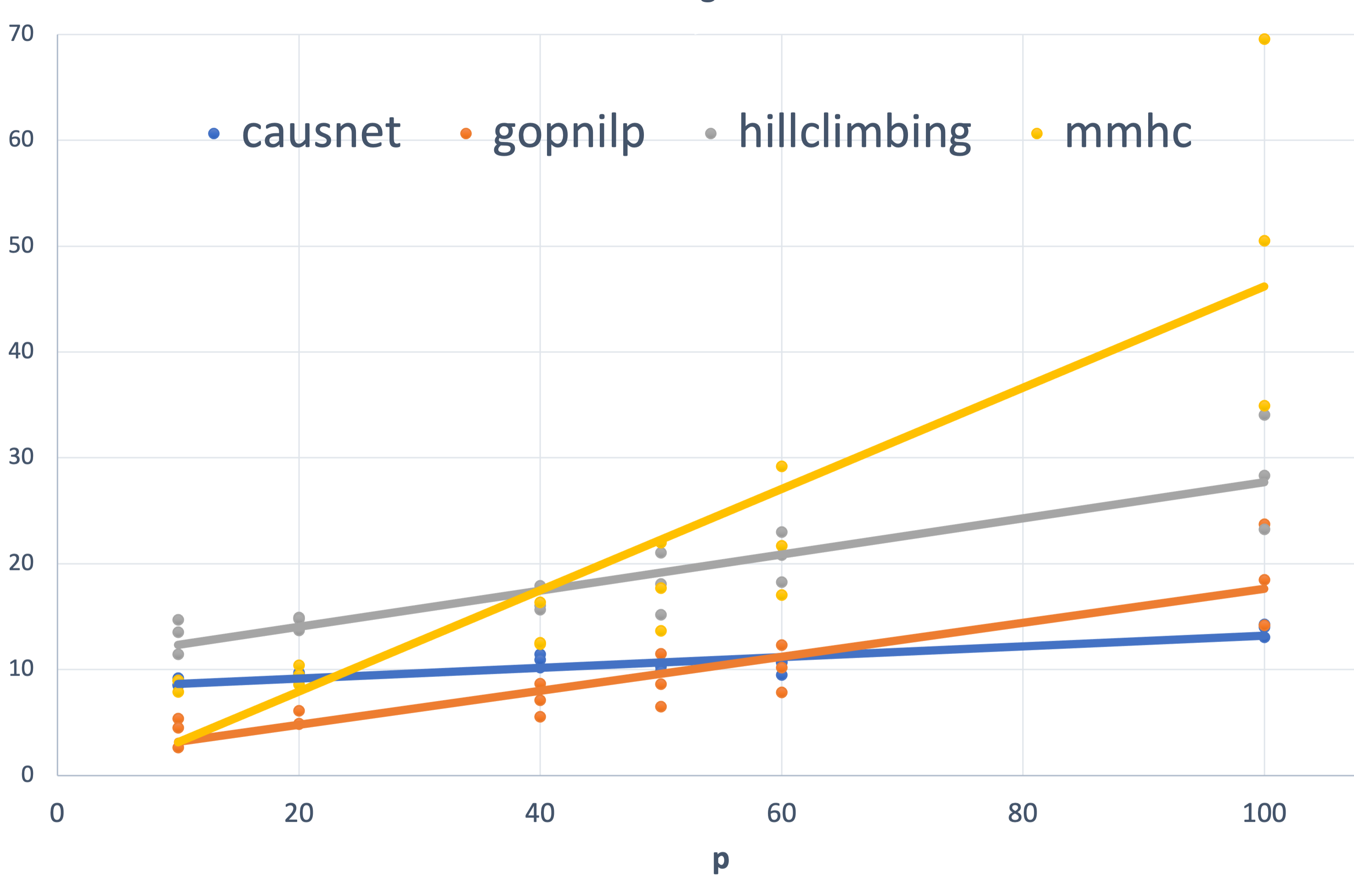}
\end{center}
\caption{\csentence{Average Hamming Distance - upto $100$ variables.}
$p = 10,20,40, 50, 60, 80, 100$ and $N = 500,1000,2000$.}
\label{HammingPic}
\end{figure} 

For the number of variables up to $40$ (Table \ref{table:fdr1}), on the metric of FDR, CausNet performs second best after Gobnilp for directed graphs and the best for undirected graphs. This is true for both scoring methods - BIC and BGE. For the number of variables between $50$ and $100$ (Table (Table \ref{table:fdr2})), our method performs the best for both directed and undirected graphs, using either of the two scoring methods, BIC and BGE. 

Figure \ref{HammingPic} shows the plot of average Hamming Distance for different values of $p$ and $N$, and their linear trend with BIC scoring for CausNet. We can see that for lower values of $p$, Gobnilp has the lowest values of Hamming Distance with CausNet the second best for most part, but CausNet performs the best for higher values of $p$. The results for the BGe scoring for CausNet are the same qualitatively.

\subsection{Phenotype based search}

For phenotype based parent set identification, we find $3$ levels of possible parents of the outcome variable. For these simulations, we use $p = 10,20,40, 50, 60, 80, 100$ and $N = 500,1000,2000$. The results are shown in figures \ref{FDRPicPheno} and \ref{HammingPicPheno}. 

\begin{figure}[!h]
$\includegraphics[width=4in]{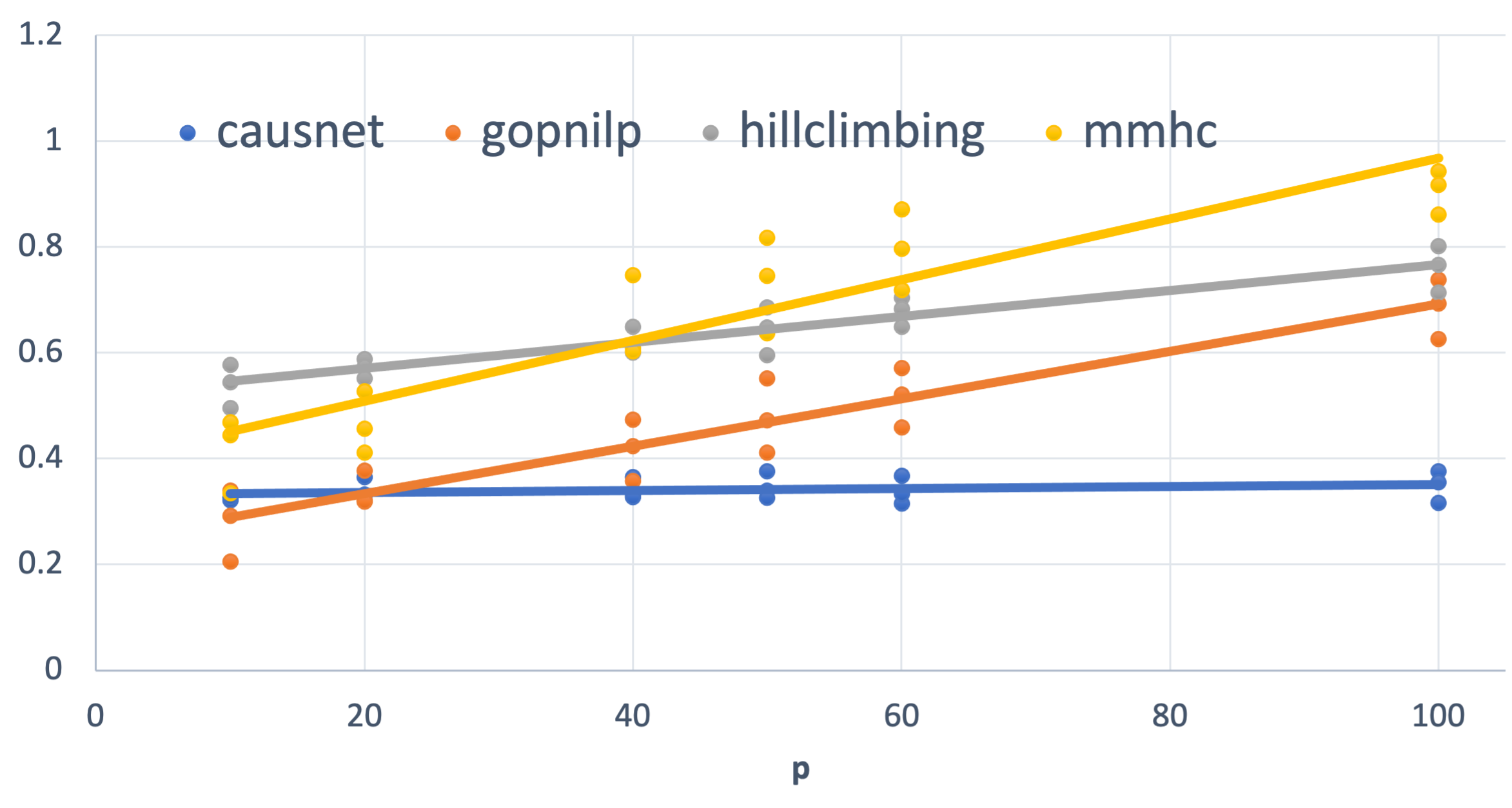}$
\caption{\csentence{Average FDR - Phenotype based search.}
$p = 10,20,40, 50, 60, 80, 100$ and $N = 500,1000,2000$.}
\label{FDRPicPheno}
\end{figure} 

\begin{figure}[h!]
$\includegraphics[width=4.5in]{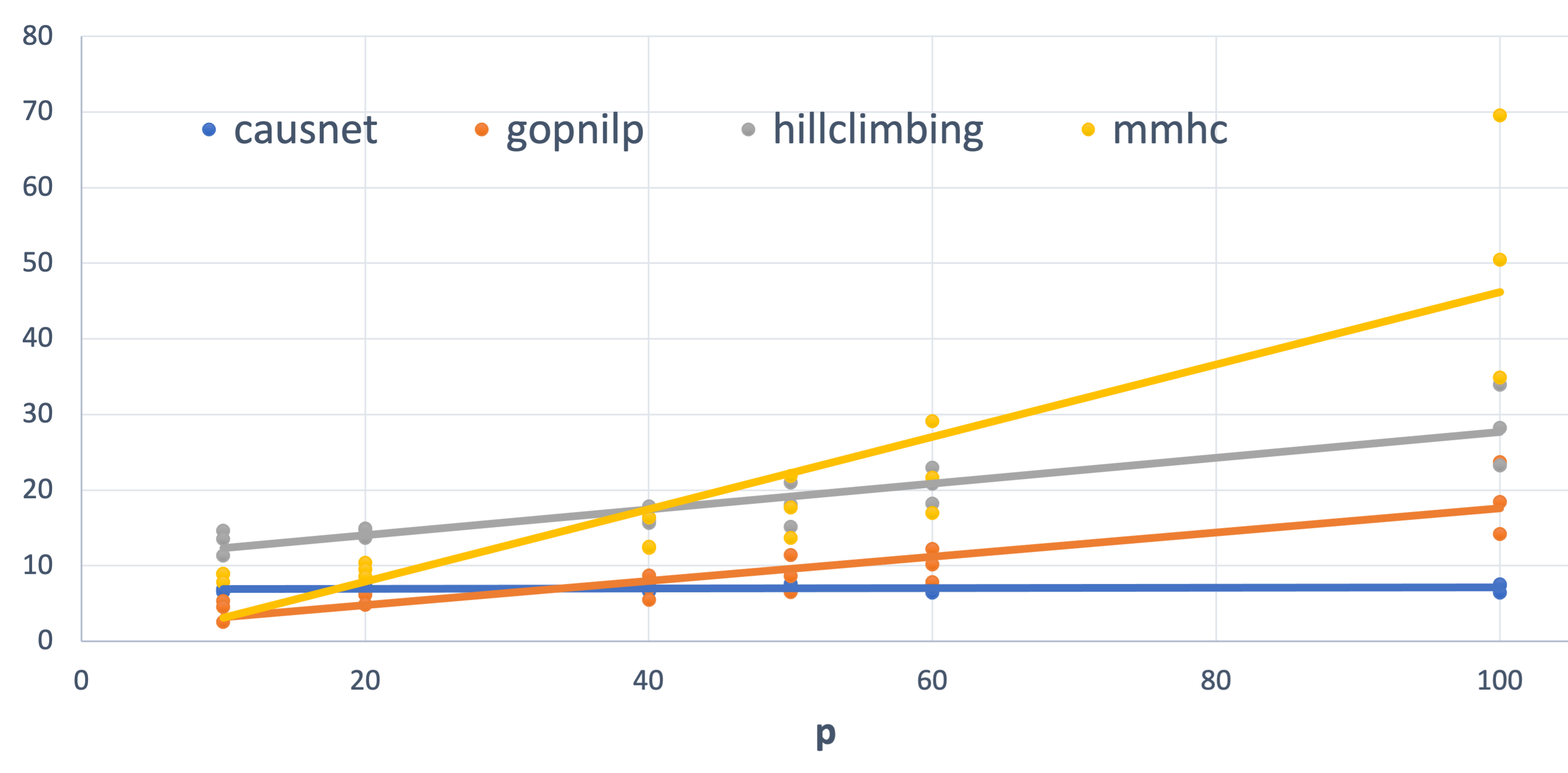}$
\caption{\csentence{Average Hamming Distance - Phenotype based search.}
$p = 10,20,40, 50, 60, 80, 100$ and $N = 500,1000,2000$..}
\label{HammingPicPheno}
\end{figure}

The FDR for phenotype-driven parent sets is the best for number of variables greater than $10$; Gobnilp is the second best. In terms of Hamming distance too, Causnet again is the best for variables more than $20$; Gobnilp and MMHC have lower Hamming distance$p$ than that of Causnet for variables less than $20$, but rise quickly for variables more than $20$; overall, Gobnilp is again the second best.

\subsection{Number of variables up to 1000}
For these simulations, we use $p = 200,500,1000$, and $N = 500,1000,2000$. For these simulations, we don't compare with the other three algorithms as they either can not handle such high number of variables or take many orders of magnitude longer. The results are shown in figures \ref{FDR1000} and \ref{Ham1000}. Observe that both the FDR and Hamming Distance values are better than what we had with other methods when the number of variables was less than $100$. 

\begin{figure}[h!]
$\includegraphics[width=4in]{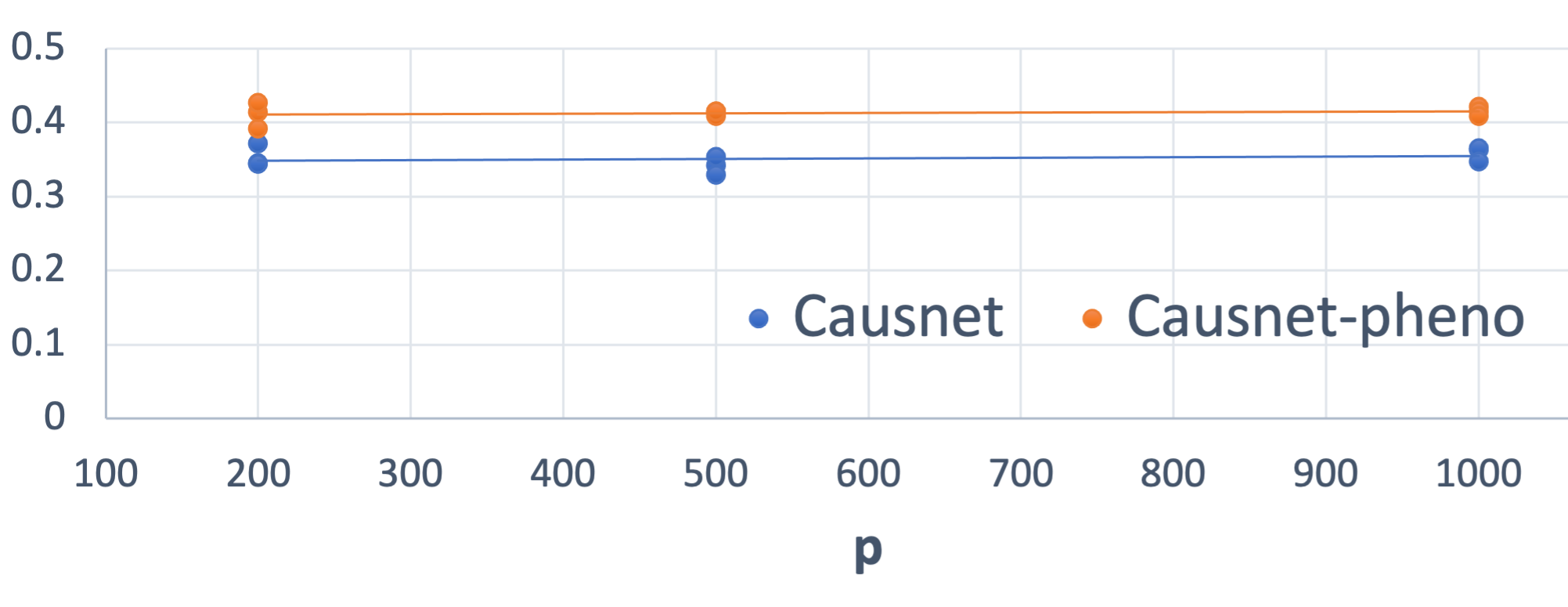}$
\caption{\csentence{Average FDR - upto $1000$ variables.}
$p = 200, 500, 1000$ and $N = 500,1000,2000$.}
\label{FDR1000}
\end{figure} 

\begin{figure}[h!]
$\includegraphics[width=4in]{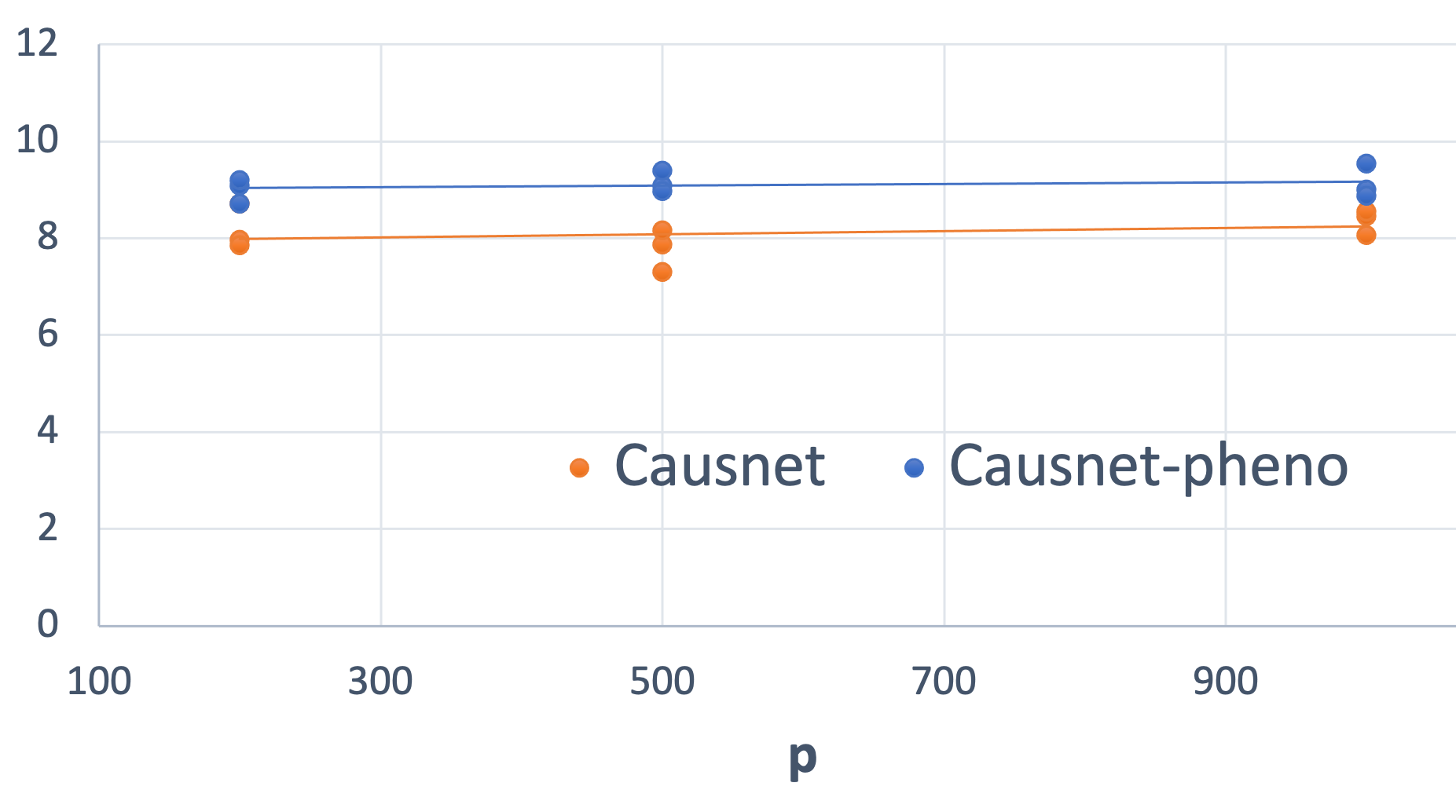}$
\caption{\csentence{Average Hamming Distance - upto $1000$ variables.}
$p = 200, 500, 1000$ and $N = 500,1000,2000$.}
\label{Ham1000}
\end{figure} 

\section{Application to clinical trial data}
We applied our method to recently published data on gene expression in relation to ovarian cancer prognosis among participants in a clinical trial
\cite{MILLSTEIN20201240}. 
The aim of this study was to develop a prognostic signature based on gene expression for overall survival (OS) in patients with high-grade serous ovarian cancer (HGSOC). Expression of 513 genes, selected from a meta-analysis of 1455 tumours and other candidates, was measured using NanoString technology from formalin-fixed paraffin-embedded tumour tissue collected from 3769 women with HGSOC from multiple studies. Elastic net regularization for survival analysis  was applied to develop a prognostic model for 5-year OS, trained on 2702 tumours from 15 studies and evaluated on an independent set of 1067 tumours from six studies. Results of this study showed that expression levels of 276 genes were associated with OS (false discovery rate $< 0.05$) in covariate-adjusted single-gene analyses. 

We applied our method CausNet to this gene expression dataset of $513$ genes and survival outcome. For dimensionality reduction and parent set identification, we used three-level phenotype driven search. For the disease node `Status',  we carried out Cox propotional hazard regression for each gene separately, adjusted for age, stage, and stratified site. 
Then we computed analysis of variance  tables for the fitted model objects and created a list of $p$-values based on $\chi^2$ distribution. We then adjusted these values using the Benjamini \& Hochberg (BH) method. Then we choose $5$ genes - ``ZFHX4" , ``TIMP3", ``COL5A2", ``FBN1", and   ``COL3A1" with the most significant p-values as possible parents of the disease node. At the next level, we used correlation between these and rest of the genes to pick their possible parent sets. After that, once again, parent sets were identified for these possible grand-parents of the disease node. These three levels of ancestors of the disease node picked $16$ genes in all with non-null parent sets. This reduced dataset and list of possible parents was used by CausNet with the BIC score to find a best network. Using the in-degree of $2$, we identified a best network as shown in Figure \ref{OvarianPic}.

\begin{figure}[!htb]
\begin{center}
\includegraphics[width=4in]{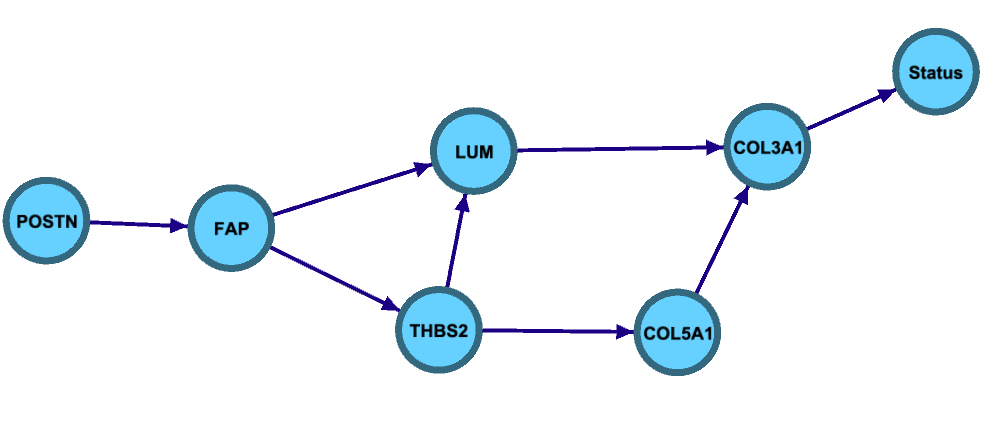}
\end{center}
\caption{Ovarian Cancer Network}
\label{OvarianPic}
\end{figure} 

\section{Discussion}
We implemented a dynamic programming based Optimal Bayesian Network (BN) Structure Discovery algorithm with parent set identification with `generational orderings' based search for optimal networks, which is a novel way to efficiently search the space of possible networks given the possible parent set.

Our main novel contribution aside from providing software is the revision of the SM algorithm 3 \cite{Silander} to incorporate possible parent sets and `generational orderings' based search for a more efficient way to explore the search space as compared to the original approach based on lexicographical ordering. In doing so, we  cover the entire constrained search space without searching through networks that don't conform to the parent set constraints. While the basic algorithm can be applied to any dataset from any domain, the phenotype based algorithm is particularly suitable for disease outcome modeling. 

The simulation results show that our algorithm performs very well when compared with three state-of-art algorithms that are widely used currently. The parent set constraints reduce the search space , and the runtime significantly, still giving better results than these algorithms, especially for number of variables greater than $60$.

The application to the recently published ovarian cancer gene-expression data with survival outcomes showed our algorithm's usefulness in disease modeling. It yielded a sparse network of $6$ nodes leading to the disease outcome from gene expression data of $513$ genes in just a few minutes on a basic personal computer. This disease pathway can be used to gain useful insights for designing further medical and biological experiments.

Other features of the algorithm include specifiable parameters - correlation, FDR cutoffs, and in-degree - which can be tuned according to the application domain. Availability of two scoring option - BIC and Bge - and implementation of survival outcomes and mixed data types makes our algorithm very suitable for many types of biomedical data, e.g. GWAS and omics data to find disease pathways.

\begin{backmatter}

\section*{Funding}
Funded by a Program Project Grant from the National Cancer Institute, P01 CA58959. 
\section*{Acknowledgements}

The authors would like to thank Professor Duncan C. Thomas for helpful inputs during the development of this work.

\section*{Availability of data and materials}
The CausNet software package in R is available at https://github.com/nand1155/CausNet.


\bibliographystyle{bmc-mathphys} 
\bibliography{proposal}      

\end{backmatter}
\end{document}